\documentclass[runningheads]{llncs}
\usepackage{graphicx}
\usepackage{algorithm2e}
\usepackage[backend=bibtex]{biblatex}
\usepackage{amsmath}
\usepackage{amssymb}
\usepackage{comment}
\usepackage{slashbox}
\begin{document}
\title{Heuristic Algorithms for the Approximation of Mutual Coherence}
%
%
\author{Gregor Betz\inst{1}
\and
Vera Chekan\inst{1,2}\orcidID{0000-0002-6165-1566} 
\and
Tamara Mchedlidze \inst{1,3}\orcidID{0000-0002-1545-5580}
}
\authorrunning{G. Betz et al.}
%
\institute{Karlsruhe Institute of Technology, Karlsruhe, Germany \\
\email{gregor.betz@kit.edu}
\and{
Humboldt-Universität zu Berlin, Germany \\
\email{vera.chekan@informatik.hu-berlin.de}
}
\and
Utrecht University, Utrecht, The Netherlands\\
\email{t.mtsentlintze@uu.nl} 
}
\maketitle              
\begin{abstract}
Mutual coherence is a measure of similarity between two opinions. Although the notion comes from philosophy, it 	is essential for a wide range of technologies, e.g., the Wahl-O-Mat system. 
In Germany, this system helps voters to find candidates that are the closest to their political preferences.
The exact computation of mutual coherence is highly time-consuming due to the iteration over all subsets of an opinion. Moreover, for every subset, an instance of the SAT model counting problem has to be solved which is known to be a hard problem in computer science. 
This work is the first study to accelerate this computation. 
We model the distribution of the so-called confirmation values as a mixture of three Gaussians and present efficient heuristics to estimate its model parameters. 
The mutual coherence is then approximated with the expected value of the distribution.
Some of the presented algorithms are fully polynomial-time, others only require solving a small number of instances of the SAT model counting problem. The average squared error of our best algorithm lies below 0.0035 which is insignificant if the efficiency is taken into account. Furthermore, the accuracy is precise enough to be used in Wahl-O-Mat-like systems.

\end{abstract}

\section{Introduction and Motivation}\label{sec:introduction_and_motivation}
A widely studied question in Bayesian epistemology is the internal coherence of two opinions, that is how consistent the opinion is or to which degree its ``parts'' support each other (the interpretation depends on the used measure of internal coherence). This notion can be generalized to the mutual coherence of two opinions. The computational issues of internal and mutual coherence are the same and in this paper, we focus on mutual coherence. 
Consider two pairs of opinions as an example \cite{Betz2019a}:
	\begin{enumerate}
		\item Opinion A: ``You should not eat animal products.''
		\item Opinion B: ``Animals have a right to life. Killing them in order to eat them is morally wrong.''	
	\end{enumerate}
	\begin{enumerate}
		\item Opinion A: ``You should not eat animal products.''
		\item Opinion B: ``Everyone can decide for themselves what they eat.''
	\end{enumerate}		
 
Intuitively, in the first pair, opinions seem to support each other whereas the second pair of opinions seem to be contradictory. For this reason, a (meaningful) measure of mutual coherence would assign a higher value to the first pair of opinions. 

There is no clear answer to the question ``When are two opinions coherent?''. 
Over the past twenty years, a multitude of coherence measures has been developed and studied. 
Most of the existing coherence measures are probabilistic, i.e., they assume that a prior probability distribution $P$ over the set of statements is given. For example, the first well-known coherence measure was proposed by Shogenji in~1999: 
\[
	C_1(A,B) = \frac{P(A \mid B)}{P(A)} = \frac{P(A \land B)} {P(A)\times P(B)}
\]
where $A, B \subseteq 2^S$ are the opinions~\cite{shogenji1999coherence}. This simple coherence measure has an obvious problem: once the opinions contain contradictory statements, the coherence value is equal to zero no matter how similar the remaining statements are. To resolve this problem, various coherence measures have been suggested. A wide overview of these measures can be found in~\cite{Douven2007a}. 
Unfortunately, all of them have further disadvantages.  It has even been proven that no ``perfect'' coherence measure (i.e., satisfying a set of the desired criteria) exists~\cite{Schippers2014a}.

One particular disadvantage of the probabilistic coherence measures is the assumption that a probability distribution $P$ over the set of statements is given as a part of the input. 
This risks to render the coherence measures utterly subjective.
To avoid this problem, a measure of mutual coherence based on the \#SAT problem was introduced by Betz~\cite{Betz2012a}. The definition of this measure will be presented in the next section. Instead of assuming a probability distribution, this measure relies on the so-called argument maps. An argument map is a graph-like representation of arguments, where an argument consists of a set of premises and a conclusion. The main problem of this measure is that the exact computation is highly time-consuming. 

Although mutual coherence is a theoretical concept in philosophy, it finds its application in the real world. In Germany, one of the most prominent examples is ``Wahl-O-Mat''. This system helps voters to find candidates that are close to their political preferences.
For this, candidates and voters fill in the same survey answering whether they agree or disagree with statements in it (it is also possible to abstain). The list of answers is then a simplified representation of the individual's political position.
To compute the similarity between two positions, a simple rule based on the Hamming-distance is used: for every statement, the same answer is scored with a positive value and distinct answers are scored negatively. The more points are scored, the more similar the positions are.
Finally, a set of candidates whose positions are the most similar to the voter is suggested. 

Due to its simplicity, this metric has a significant disadvantage: it ignores possible inferential relations between single statements.
With better coherence measures, a better quality of the suggestions can be achieved.  
Since the ``Wahl-O-Mat'' system is used by a multitude of voters, it has to perform highly efficiently. Especially, the mutual coherence needs to be calculated in a feasible time. 

The measure of mutual coherence that we study here is more sophisticated than the above-mentioned simple rule: using it in systems like ``Wahl-O-Mat'', would improve their quality. 
But since the naive computation of this measure is very time-consuming, it can not be used in the system so far.
And this motivates the goal of this project: we develop heuristics for the efficient approximation of this measure of mutual coherence.

The paper proceeds as follows: In~Section~\ref{sec:preliminaries}, we introduce required definitions and present the considered measure of mutual coherence. Next,~in~Section~\ref{sec:synthetic}, we motivate the need for synthetic argument maps and present an algorithm to create them. Then, in~Section~\ref{sec:heuristics}, we present the main contribution of this project, i.e., the heuristics for the efficient approximation of mutual coherence. After that, in~Section~\ref{sec:results}, the accuracy of the heuristics is demonstrated in experiments with synthetic data and finally, in~Section~\ref{sec:conclusion}, we summarize the results of the work.

\section{Preliminaries}\label{sec:preliminaries}

Here, we formally define the studied measure of mutual coherence and show that its naive computation is indeed highly time-consuming. 

First, we introduce the \textit{\#SAT} problem, also called the \textit{SAT model counting} problem. The instance of the problem is a boolean formula or, equivalently, it is a set of clauses (a clause is a logical disjunction, e.g., $a \lor \lnot b \lor c$). In the better-known \textit{SAT} problem, the question is whether there exists a truth-value assignment under which the formula becomes true. The SAT model counting problem generalizes the question and asks how many truth-value assignments make the formula true. Both problems are known to be hard in computer science. To be more precise, SAT is $\mathcal{NP}$-complete and \#SAT is \#$\mathcal{P}$-complete. In other words, it is conjectured that these problems admit no polynomial-time algorithms. Nevertheless, there exist diverse heuristics that are in practice efficient for many instances. The state-of-the-art \#SAT-solvers combine these heuristics and take into account both theoretical (e.g., unit propagation) and technical~(e.g., cache-efficiency) details. In our project, for model counting, we use the ganak model counter~\cite{SRSM19}. This implementation won the ``Model Counting Competition 2020'' and it belongs to the state-of-the-art model counters.

Let $S$ be a set of sentences (equivalent to boolean variables in terms of SAT) so that the set is closed under negation and let $N := |S|/2$. An \textit{argument}~$a$ is a tuple $(P_a, c_a)$ with $P_a \subseteq S$ and~$c_a \in S$ where $P_a$ denotes the set of \textit{premises} and~$c_a$ is the \textit{conclusion}. 
Let $I$ be a set of arguments. 

A pair $(S, I)$ is called a \textit{simple structured argumentation framework (SSAF)}. SSAFs can be visualized with \textit{Argument Maps (AMs)} that represent \textit{support} (green) and \textit{attack} (red) relations between statements
(see Figure~\ref{fig:argmap}). Every line in an AM corresponds to an argument. For example, consider the leftmost green line in Figure~\ref{fig:argmap}. The premises of the argument are ``Other animals eat meat.'' and ``It is okay for humans to do things if other animals do them.'' 
The conclusion is ``It is okay for humans to eat meat''. Arguments represent logical relations between statements. 
If a (consistent) speaker agrees with all premises of the argument, then he also agrees with its conclusion.
A red line denotes that agreeing with all premises requires disagreeing with the conclusion (e.g., the rightmost line in Figure~\ref{fig:argmap}).
This representation is a hyper-graph but it can be easily transformed into a (simple) graph by introducing a dummy-vertex for every hyper-edge.
This graph contains the whole logic of an SSAF and instead of working with the set $I$ of tuples $(P_a, c_a)$, it is possible to utilize the graph structure of AMs.

\begin{figure}[t]
    \centering
    \includegraphics[scale=0.6]{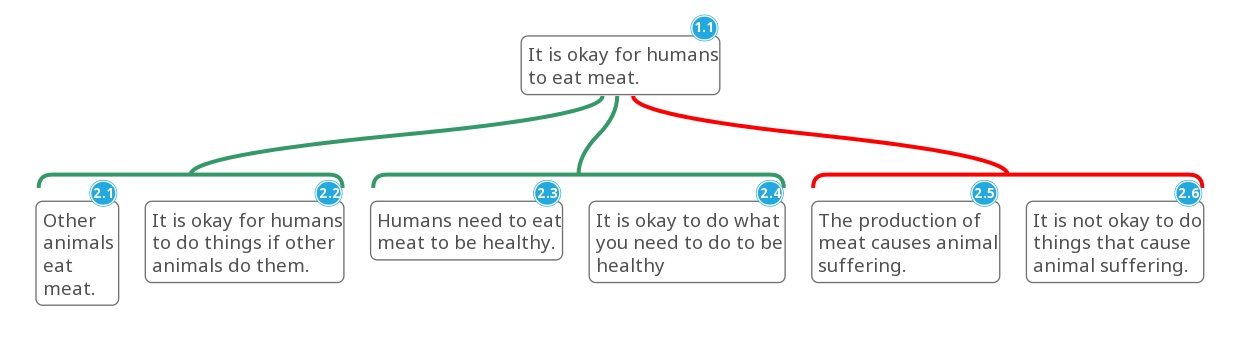}
    \caption{
        The support (attack) relations between arguments are depicted green (red). Source: \protect\url{https://github.com/mermaid-js/mermaid/issues/747}
    }
    \label{fig:argmap}
\end{figure}

A \textit{position} is a truth-value assignment to a set~$S_D \subseteq S$:
\[
	A: S_D \rightarrow \{\text{True, False}\}
\]
In the following, we use the terms ``opinion'' and ``position'' interchangeably. For simplicity, we sometimes define the position by the set of statements assigned the True-value and write $A \subseteq S$.

Position $A$ is \textit{complete} if its domain is $S$.  
We say that complete position $A$ \textit{extends} position~$B$ if positions $A$ and $B$ agree on values assigned to the domain of $B$. 
Complete position $A$ \textit{does not extend}~$B$ otherwise, i.e., if the positions disagree on at least one statement in the domain of $B$.
Complete position $A$ is \textit{consistent} if 
the following holds:
\begin{enumerate}
\item $\forall s \in S: A(s) \neq A(\lnot s)$
\item $\forall a = (P_a, c_a) \in I: \left (\left (\forall p \in P_a: A(p) = \text{ True}\right) \rightarrow A(c_a) = \text{ True} \right)$
\end{enumerate}
We denote the number of complete consistent positions extending (not extending) $A$ with $\sigma_A$ ($\sigma_{\lnot A}$). Note that these values depend on the set of arguments~$I$.
Similarly, for two positions $A$ and~$B$, we denote the number of complete consistent positions extending both $A$ and~$B$ with $\sigma_{A, B}$. 
Here, the \#SAT model counting problem appears: for example, $\sigma_A$ is the number of truth-value assignments of $S$ such that the arguments $I$ and statements of $A$ are true.
We write~$Y \vDash_I X$ ($Y \vDash_I \lnot X$) if every complete consistent position extending $Y$ extends (does not extend) $X$. Thus, $\vDash_I$ denotes the logical implication in terms of SAT.

The value 
\[
	DOJ(A\mid B) \stackrel{\text{def}}{=} \sigma_{A, B} / \sigma_B 
\]
is called the \textit{degree of justification} of $A$ by $B$.

The \textit{Kemeny-Oppenheim confirmation measure} is defined as follows: 
\[
	Conf(X, Y) \stackrel{\text{def}}{=}
	\begin{cases}
		\frac{DOJ(Y\mid X) - DOJ(Y\mid \lnot X)}{DOJ(Y\mid X) + DOJ(Y\mid \lnot X)} & \text{ if } Y \nvDash_I X \land Y \nvDash_I \lnot X \\
		1 & \text{ if } Y \vDash_I X \\
		-1 & \text{ if } Y \vDash_I \lnot X \\
	\end{cases}
\]

Finally, we can define the mutual coherence between positions $A$ and $B$ as introduced by Betz et al.~\cite{Betz2019a}:
\[
	MutCoh(A, B) \stackrel{\text{def}}{=} \frac{1}{2\cdot(2^{k_A} - 1)} \sum_{\emptyset \neq X \subseteq A}Conf(X, B) + \frac{1}{2\cdot(2^{k_B} - 1)} \sum_{\emptyset \neq X \subseteq B}Conf(X, A). 
\]
In the formula, $k_A$ and $k_B$ denote the domain size of $A$ and $B$, respectively.

In the following, writing about the mutual coherence we always refer to the measure $MutCoh(\cdot, \cdot)$.
Consider the straightforward computation of the mutual coherence. It requires an exponential (in the size of the opinion) number of iterations. In every iteration, we solve a constant number of \#SAT instances which in turn requires exponential~(in $|S|$) time in the worst-case. 
What concerns the second component of the running time, using state-of-the-art model counters would speed it up. 
Since this part only depends on the model counter, we will not further consider it in this paper.
However, the number of required iterations does not depend on the model counter. 
In this project, we focus on it and reduce the number of subsets taken into account and hence, the number of needed runs of the model counter (see Section~\ref{sec:heuristics}).

Finally, we observe that formula for the mutual coherence consists of two symmetric summands. For this reason, we introduce the \textit{one-sided coherence}:
\[
	OneCoh(A, B) \stackrel{\text{def}}{=} \frac{1}{2^{k_A} - 1} \sum_{\emptyset \neq X \subseteq A}Conf(X, B)
\]
which is the average confirmation value over the non-empty subsets of $A$. Then
\[
	MutCoh(A, B) = \frac{1}{2} \left( OneCoh(A, B) + OneCoh(B, A) \right)
\]

During the project, we have found out that it is meaningful to consider $OneCoh(A, B)$ and $OneCoh(B, A)$ separately. So in the following, we are only interested in the efficient approximation of $OneCoh(\cdot, \cdot)$. This yields then a still efficient approximation of $MutCoh(\cdot, \cdot)$.

\section{Synthetic Argument Maps}\label{sec:synthetic}

An important property of approximation heuristics is their scalability. In our case, this is the ability to stay accurate for growing instances.
There are two parameters of interest: the size of the AM and the size of the opinion. 
Particularly, we are interested in AMs of arbitrary size. 
During the project, we encountered the problem of missing test data. Although a huge number of AMs can be found in sources like \url{aifbd.org}, the absolute majority of them are small (i.e., consist of at most 25 statements). For such AMs, both problems (model counting and the computation of mutual coherence) can be solved in a feasible time and hence, there is no need for the approximation. Larger AMs exist and they are very important for the research but there are fewer of them and they are difficult to find in the public domain. 
For this reason, we have developed an algorithm for the creation of synthetic AMs. 
Mutual coherence is not the only area where AMs appear, so the algorithm can also be used beyond our project. 
Here, we sketch the process and a pseudocode representation is provided in Appendix~\ref{appendix:pseudocode}. 

The input of the algorithm $n, k, \alpha, \psi, \gamma,d$ is the set of parameters steering the size and the shape of the created AM.
The algorithm is randomized: run with the same input, it produces different AMs.
First of all, $n$ is the number of statements. The parameter $\alpha$ determines the number of arguments $\alpha \cdot n$: we observed that due to the tree-likeness of AMs, the number of arguments can be approximated as a linear function in the number of statements. The second parameter specifying the density of the AM is $d$. It is a probability function mapping an integer number $k$ to the average fraction of arguments with exactly $k$ premises in an AM. This number needs to be computed from an AM (or a set of AMs) which will then be mimicked by the algorithm. For example, we used:
\[
	d = \{(2, 0.19), (3, 0.23), (4, 0.32), (5, 0.26)\}
\]
This distribution was calculated from the AM of Veggie Debate\footnote{\url{debatelab.philosophie.kit.edu/sm_talk-daprpc_89.php}}. The remaining parameters $\psi, \gamma$, and $k$ control the creation of an AM. Parameter $k$ is the number of so-called \textit{key statements} which initialize the AM. They model the most important statements of the AM on top of which supporting and attacking arguments are built. 
The number~$k$ is typically small because debates are individuated with reference to central questions and key statements that are discussed. The debates documented in the Debater's Handbook, for example, mostly evolve around 2-3 key claims \cite{newman2013pros}. 

A \textit{level} of a statement in the current AM is defined as follows.
For a literal~$l$, let $Var(l)$ be the corresponding variable, i.e.,
\[
	Var(v) = Var(\lnot v) = v.
\] 
A key statement $s$ has level 0:
\[
	Level(s) = 0.
\]
Otherwise, for a statement $s$, we consider the set conclusions $C(s)$ of all arguments in which~$s$ occurs as a premise:
\[
	C(s) = \{ c \mid p_1 \land \dots \land p_k \rightarrow c \in I', s \in \{Var(p_1), \dots, Var(p_k)\} \}
\]
where $I'$ denotes the current set of arguments. Then the level of $s$ is defined by:
\[
	Level(s) = \min_{c \in C(s)} Level(c) + 1
\]
In other words, the level of a statement $s$ is the length of the shortest path from~$s$ to some key statement minus the number of arguments on this path.
Note that the level might change after an argument is added to the AM and the level is only defined for statements already added to the AM. 

The arguments are generated one by one according to the following scheme:
\begin{enumerate}
	\item To ensure that the arguments are built on top of the central statements, we pick a conclusion from the statements that are already in the AM: initially, these are only the key statements. The probability to pick statement $s$ is proportional to $\psi ^ {Level(s)}$. 
	After the conclusion is chosen, it is negated with a probability of~$0.5$ to create an attacking argument.
	\item Next, the number of premises $t$ is picked according to $d$.
	\item After that, each of $t$ premises is picked independently.
	A statement $s$ is picked with probability proportional to $\gamma ^ {Arguments(s)}$, where $Arguments(s)$ denotes the number of arguments in which $s$ occurs. Especially, the statements that have not been used in arguments yet obtain the largest probability.
	\item Finally, we check if the AM stays satisfiable with the newly created argument, that is, there exists a complete consistent position with respect to the current AM.
	If so, the argument is added to the AM. Otherwise, the argument is dismissed and the process is repeated. To avoid infinite loops, a termination criterion is added: for example, restart the algorithm or return an error if there were too many unsuccessful attempts to create an argument. 
\end{enumerate}
The pseudocode can be found in Appendix~\ref{appendix:pseudocode}.

\section{Approximation of the One-Sided Coherence}\label{sec:heuristics}

In this section, we present the heuristics for the approximation of the one-sided coherence: 
\[
	OneCoh(A, B) = \frac{1}{2^{k_A} - 1} \sum_{\emptyset \neq X \subseteq A}Conf(X, B). 
\]
To solve the problem, we have taken a closer look at the definition and the interpretation of the confirmation measure $Conf(\cdot, \cdot)$. A confirmation value is a real number between -1 and 1. We recall the meaning of these extreme values.
It holds:
\[
	Conf(X, B) = -1 \Leftrightarrow B \vDash_I \lnot X,
\] 
or in other words: the union of $B$ and $X$ is logically contradictory (if we assume that the arguments $I$ hold). 
A simple special case occurs when~$X$ contains a literal $s$ such that $\lnot s$ belongs to $B$. In this case, $B$ and $X$ are contradictory regardless of the set of arguments $I$. 

Similarly, it holds:
\[
	Conf(X, B) = 1 \Leftrightarrow B \vDash_I X,
\]
that is, $B$ logically implies $X$ (if we again assume that the arguments $I$ hold). Inter alia, this happens when $X \subseteq B$: Regardless of the set of arguments~$I$, the position $B$ logically implies its subsets. 

This results in two types of subsets for which the confirmation values can be determined without running the model counter. This motivates the following lemma.

\begin{lemma}\label{lemma:simple_bound}
	Let $A, B \subset S$ be positions. We set: 
	\[
		Neg(A, B) \stackrel{\text{def}}{=} \{ x \in A \mid \lnot x \in B\}
	\]
	\[
		Com(A, B) \stackrel{\text{def}}{=} A \cap B
	\]
	Then:
	\begin{enumerate}
		\item For every $X \subseteq A$ with $X \cap Neg(A,B) \neq \emptyset$: $Conf(X, B) = -1$. And it holds: 
		\[
			\big| \{ X \mid X \subseteq A, X \cap Neg(A,B) \neq \emptyset \} \big| = 2^{|A|} - 2^{|A| - |Neg(A,B)|}
		\]
		
		\item For every $\emptyset \neq X \subseteq Com(A, B)$: $Conf(X, B) = 1$. And it holds: 
		\[
			\big| \{X \mid \emptyset \neq X \subseteq Com(A, B)\} \big| = 2^{|Com(A, B)|} - 1
		\]
	\end{enumerate}
\end{lemma}

\begin{proof}
We prove the claims separately.
\begin{enumerate}
\item Consider $X \subseteq A$ with $X \cap Neg(A, B) \neq \emptyset$. 
There exists $x \in X \cap Neg(A, B)$. 
For an arbitrary complete consistent position $P$ extending $B$, it holds $P(x) = false$ since $\lnot x \in B$. 
Because of $x \in X$, the position $P$ does not extend $X$. 
Thereby,~$B \vDash_I \lnot X$ and hence, $Conf(X, B) = -1$. 
The subsets of~$A$ which are disjoint from $Neg(A,B)$ are exactly the subsets of $A \setminus Neg(A, B)$ and hence, there are $2 ^ {|A| - |Neg(A, B)|}$ of them.

\item Consider $\emptyset \neq X \subseteq Com(A, B)$. By the definition of $Com(A, B)$, we have:~$X \subseteq B$. So every complete consistent position extending $B$ extends $X$ too and hence,~$B \vDash_I X$ and~$Conf(X, B) = 1$. There are $2^{|Com(A, B)|} - 1$ non-empty subsets of $Com(A, B)$. \qed
\end{enumerate}
\end{proof}

This lemma provides lower bounds on the number of subsets with confirmation values -1 and 1. To calculate these bounds, subsets $Com(A,B), Neg(A,B) \subseteq A$ are required. They can be computed in $\mathcal{O}(max\{|A| \log |A|, |B| \log |B|\})$ (i.e.,~polynomial time) by first sorting the literals in opinions $A$ and $B$ and then iterating over the two sorted lists. The lemma can be strengthened to the following (the proof is analogous and can be found in Appendix~\ref{app:proof_improved_bounds}): 

\begin{lemma}\label{lemma:improved_bounds}
Let $A, B \subset S$ be positions. Let 
	\[
		Cntr(A, B) \stackrel{\text{def}}{=} \{ x \in A \mid B \vDash_I \lnot \{x\}\}
	\]
	and
	\[
		Impl(A, B) \stackrel{\text{def}}{=} \{ x \in A \mid B \vDash_I \{x\}\}
	\]
	Then:
	\begin{enumerate}
		\item For every $X \subseteq A$ with $X \cap Cntr(A, B) \neq \emptyset$: $Conf(X, B) = -1$. And: 
		\[
			\big| \{ X \mid X \subseteq A, X \cap Cntr(A, B) \neq \emptyset \} \big| = 2^{|A|} - 2^{|A| - |Cntr(A,B)|} 
		\]
		\item For every $\emptyset \neq X \subseteq Impl(A, B)$: $Conf(X, B) = 1$. And:
		\[
			\big| \{ X \mid \emptyset \neq X \subseteq Impl(A, B) \} \big| = 2^{|Impl(A,B)|} - 1
		\]
	\end{enumerate}
\end{lemma}

Observe that $Neg(A, B) \subseteq Cntr(A, B)$ and $Com(A, B) \subseteq Impl(A, B)$.
Thus, Lemma~\ref{lemma:improved_bounds} yields stronger lower bounds on the number of subsets of $A$ with confirmation values -1 and 1. However, the computation of these bounds is not necessarily polynomial anymore. One way to determine the set $Cntr(A, B)$ is to iterate over statements $x \in A$ and check if $\sigma_B \stackrel{\text{?}}{=} \sigma_{B \cup \{x\}}$ using a \#SAT solver. This is exactly the case if $x \in Impl(A,B)$. Another way is to apply a SAT solver to check if $B$ and $I$ logically imply $x$. 
Similarly, to compute the set~$Cntr(A,B)$ we iterate over statements $x \in A$ and check if $\sigma_{B\cup \{x\}} \stackrel{\text{?}}{=} 0$ or, equivalently if $B$ and $I$ logically imply $\lnot x$. 
Each of these approaches requires~$\mathcal{O}(|A|)$ runs of a \#SAT or a SAT solver, respectively.

\begin{figure}[b]
    \centering
    \includegraphics[scale=0.3]{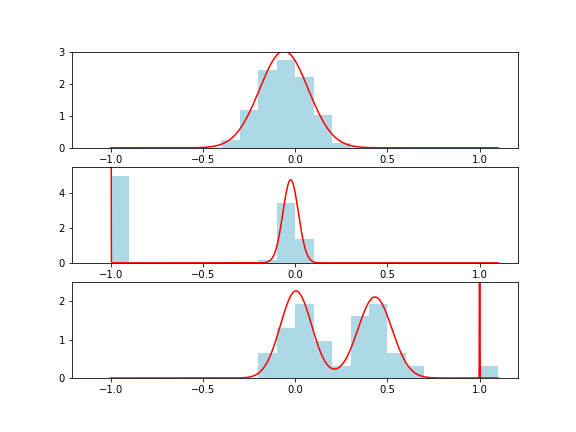}
    \caption{
        The distribution of confirmation values (on the x-axis) normally has up to three peaks. 
    }
    \label{fig:3_gaussians}
\end{figure}

These theoretical lower bounds have inspired us to look at the real distribution of confirmation values. 
We have plotted it for different opinion pairs and different AMs and observed that in most cases, the distribution has up to three peaks and therefore can possibly be approximated by a mixture of three Gaussians (see Figure~\ref{fig:3_gaussians}). 
Such a mixture is determined by the following mixture parameters: 
\begin{enumerate}
	\item Mean values $\mu = (\mu_1,  \mu_2, \mu_3)$
	\item Standard deviations $\sigma = (\sigma_1, \sigma_2, \sigma_3)$
	\item Mixture parameters $w = (w_1, w_2, w_3)$ with $w_1, w_2, w_3 \geq 0$ and $w_1 + w_2 + w_3 = 1$ 
\end{enumerate} 
Then, the probability density function is given by 
\[
	\mathcal{N}_3(w, \mu, \sigma) \stackrel{\text{def}}{=} w_1 \mathcal{N}(\mu_1, \sigma_1^2) + w_2 \mathcal{N}(\mu_2, \sigma_2^2) + w_3 \mathcal{N}(\mu_3, \sigma_3^2),
\] 
where $\mathcal{N}(\mu_i, \sigma_i^2)$ denotes the normal distribution with the mean (i.e., expectation)~$\mu_i$ and standard deviation~$\sigma_i$. 
As a result, the expected (i.e., mean) value of the distribution is given by 
\[
	\mathbb{E}(\mathcal{N}_3) = w_1\mu_1 + w_2\mu_2 + w_3\mu_3.
\]
We again recall the definition of the one-sided coherence that we want to approximate:
\[
	OneCoh(A, B) = \frac{1}{2^{k_A} - 1} \sum_{\emptyset \neq X \subseteq A}Conf(X, B). 
\]
So the one-sided coherence is the mean value of the confirmation values. 
This motivates the main approach we follow in our heuristics:
Model the distribution of confirmation values as a mixture of three Gaussians, estimate the mixture parameters $w, \mu,\sigma$, and finally, compute the mean value of the distribution which is then an approximation of the one-sided coherence. 
Note that $\sigma_1, \sigma_2, \sigma_3$ do not influence the mean value $\mathbb{E}(\mathcal{N}_3)$ and hence, it is not necessary to estimate them. 

\subsection{The Heuristics}
\subsubsection{The Estimation of Weights.}
Lemma~\ref{lemma:simple_bound} provides lower bounds on the fraction of subsets with confirmation values -1 and 1, so we can fix the leftmost and the rightmost Gaussians as follows:
\begin{enumerate}
	\item \[
		\mu_1 = -1,  w_1 = \frac{2^{|A|} - 2^{|A| - |Neg(A, B)|}}{2^{|A|} - 1}
	\]
	\item 
	\[
		\mu_3 = 1, w_3 = \frac{2^{|Com(A, B)|} - 1}{2^{|A|} - 1}
	\]
	\item  
	\[
		w_2 = 1 - w_1 - w_3
	\]
\end{enumerate}
We call it the \textit{simpler estimation} of mixture weights or just \textit{simpler weights}. 
Alternatively, the mixture weights $w_1$ and $w_3$ (and hence $w_2 = 1 - w_1 - w_3$) can be defined by replacing $Neg(A, B)$ with $Cntr(A, B)$ and $Com(A, B)$ with~$Impl(A, B)$ (see Lemma~\ref{lemma:improved_bounds}). We call it the \textit{finer estimation} of weights or just \textit{finer weights}.
We expect that the finer estimation results in better accuracy but longer running time.
Based on the weights (finer or simpler), we have developed the following four heuristics. 
\subsubsection{Direct Estimation.}

The simplest action is to set $\mu_2 = 0$.
The underlying assumption is that the subsets $X$ of $A$ which are neither contradictory to $B$ nor implied by it are just logically independent of $B$. This corresponds to $Conf(X, B) = 0$. 
We also assume that lower bounds from Lemma~\ref{lemma:simple_bound} or~\ref{lemma:improved_bounds} yield a good approximation of the fraction of subsets with confirmation values -1 and 1. 
As soon as the weights are computed, this approach requires constant time to get the approximated value. For this reason, the direct estimation with simpler weights is a fully polynomial-time algorithm. The direct estimation with finer weights has a complexity of $\mathcal{O}(|A|)$ model counting operations.

The remaining three approaches are sample-based: we sample $\beta \cdot |A|$ subsets of $A$, compute the corresponding confirmation values $V = \{v_1, \dots, v_{\beta \cdot |A|} \}$ using a \#SAT solver, and use these values to estimate the mean value of the distribution~$\mathcal{N}_3$ in different ways. Since the computation of every confirmation value requires a run of a model counter, the number of samples is linear in the size of the opinion to hold the number of these runs small.

\subsubsection{Average.}

The simplest approach is just to compute the mean value of samples and return it, so the approximation is:
\[
	mean(V) \stackrel{\text{def}}{=} \frac{1}{|V|}\sum_{v \in V} v.
\]
Note that this approach does not use the estimated weights and relies only on the set of samples $V$. For this reason, we expect that this approach will tend to require more samples to achieve the same accuracy compared with the approaches that additionally utilize the estimated weights.
The complexity is~$\mathcal{O}(\beta \cdot |A|) = \mathcal{O}(|V|)$ runs of the model counter.

\subsubsection{Average $\mu_2$.}

This approach combines the previous two as follows. We act similarly to direct estimation but instead of $\mu_2 = 0$, we set:
\[
	\mu_2 = mean(V).
\]
This way, we assume that all samples $v \in V$ belong to the middle Gaussian. The complexity of this approach consists of the weights' estimation and of the average-approach. Note that the new average-$\mu_2$-approach has the disadvantage that it might take certain subsets into account twice: first, in $w_1$ or $w_3$ and second, as a sample in $V$. This might lead to the under- or overestimation of the mean value of the distribution. Later, we will describe a filtering technique to avoid this problem. 

\subsubsection{Fit $\mu_2$.}
Another solution for this problem is to estimate the position of the second Gaussian (i.e., $\mu_2$) by applying the \textit{Expectation-Maximization (EM) algorithm}~\cite{dempster1977maximum}. 
In statistics, the EM algorithm is an iterative method to estimate the model parameters of a distribution where the model depends on unobserved variables. In our case, the unobserved variable corresponds to the Gaussian to which a point belongs.

In this way, the aforementioned subsets will not be taken into account twice: instead, the EM algorithm will ``assign'' them to the leftmost or the rightmost Gaussian, respectively. To be more precise, our adaptation of the EM algorithm runs as follows: Fix $w_1, w_2, w_3$ (simpler or finer weights), $\mu_1 = -1$, and $\mu_3 = 1$ so that they stay constant during the algorithm; The expectation step of the algorithm is the same as in the original EM algorithm; however, in the maximization step, we only change $\mu_2$ according to the EM algorithm.
We expect this algorithm to perform most accurately since it overcomes the disadvantages of the previous approaches.
The running time is the running time of the average-approach plus the running time of the EM algorithm. The latter depends on the number of iterations needed until convergence (normally a small constant) and one iteration is linear in~$|V|$ (and hence in~$|A|$).

\subsection{Filtering}
As we mentioned earlier, a disadvantage of the average-$\mu_2$-approach is that it tends to take certain subsets into account twice. The fit-$\mu_2$-approach solves this problem with the EM algorithm just assigning such samples to the leftmost or the rightmost Gaussian which are already fixed. This approach, in turn, has the disadvantage that such subsets are useless because they contain no new information. To solve both of these problems, we introduce the \textit{filtering}-technique. In the \textit{filtered} average-$\mu_2$- and filtered fit $\mu_2$-approaches, instead of sampling from all subsets of $A$, we only sample subsets $\emptyset \neq X \subseteq A$ such that (assume we use simpler weights):
\[
	X \cap Neg(A,B) = \emptyset \text{ and } X \nsubseteq Com(A,B).
\]
This way, no subset will be taken into account twice.

\section{Results}\label{sec:results}

\subsection{The Description of the Dataset}

In this section, we evaluate the accuracy of our heuristics.
We have run the algorithms on synthetic argument maps. As we mentioned at the beginning, the real argument maps are mostly small so that both the \#SAT solving and the computation of the one-sided coherence can be computed efficiently. For this reason, in this project, we were more interested in larger argument maps so and tested the approximation algorithms on synthetic data. 

The synthetic argument maps were created according to~Section~\ref{sec:synthetic} with the following parameters: 
\begin{itemize}
	\item $n \in \{50, 100, 150, 200\}$
	\item $\gamma = \psi = 0.5$
	\item $\alpha \in \{0.3, 0.5, 0.7\}$
	\item $k \in \{3, 5, 7, 10\}$
\end{itemize} 
We have tried out different combinations of these parameters in order to make the dataset as various as possible. However, for some combinations (e.g., $n = 200$ and~$\alpha \in \{0.5, 0.7\}$) the running time of model counting became infeasible (especially, if it had to be done for every subset of an opinion). For this reason, there are more argument maps with smaller $n$, and larger argument maps are created with tendentially smaller $\alpha$. 
After that, for every argument map, we create a set of opinion pairs where opinions have sizes 5, 7, and 10. For every argument map and every opinion size, we tried to create several pairs of opinions. Again, we encountered the problem that the computation of the ground-truth~(i.e., one-sided coherence) for some pairs of opinions was infeasible (i.e., the model counting took too long). For this reason, some combinations of parameters have been excluded~(e.g., for argument maps consisting of 200 statements, there are only opinions of size 5).
Totally, we have created 63 argument maps and 1627 pairs of opinions.

The plots will have the following structure. On the x-axis, we have the exact value of the one-sided mutual coherence, and on the y-axis, we have the result of a certain algorithm. Both axes have the range [-1, 1]. Thus, for an accurate approximation algorithm, the points fit the line $x = y$; the line is also shown for convenience.
More details will be provided later.

\begin{figure}[t]
    \centering
    \includegraphics[scale = 0.13]{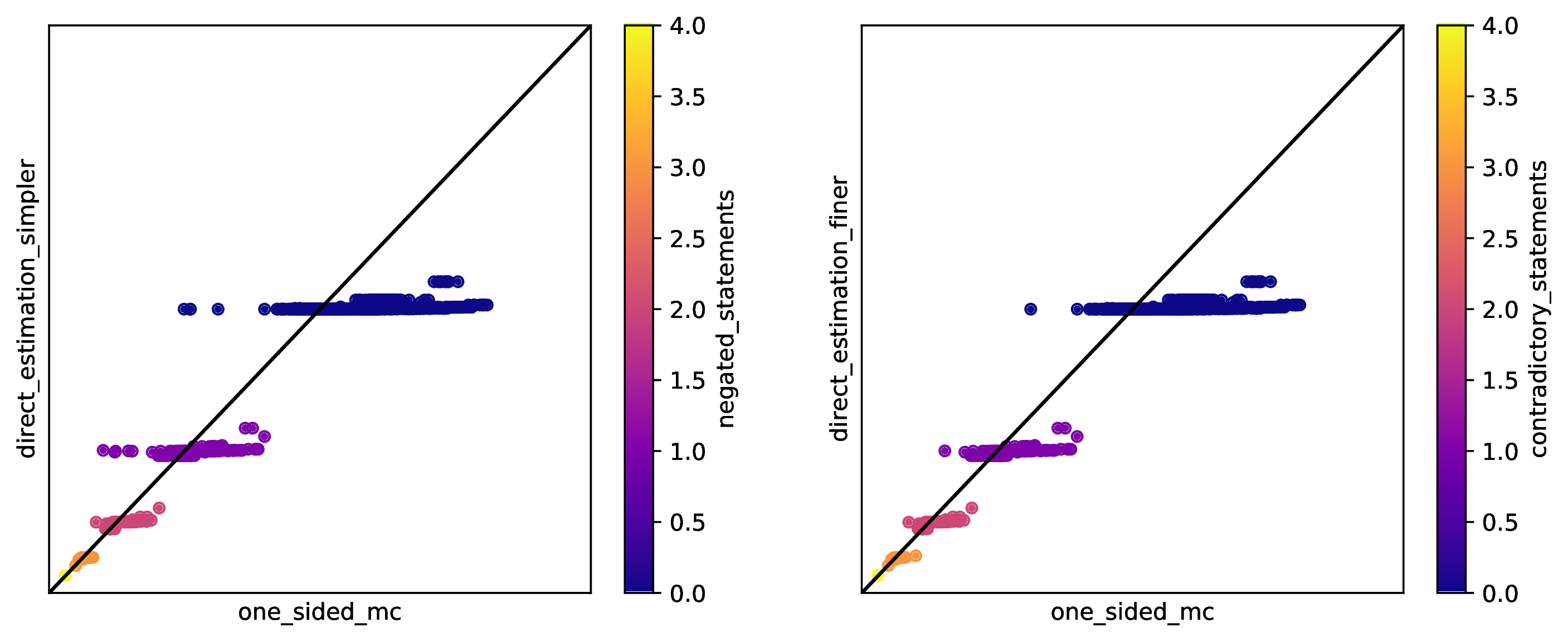}
    \caption{
    	Direct estimation of one-sided coherence: simpler weights on the left and finer weights on the right. The color denotes the number of negated or contradictory statements in the opinion pair, respectively.
    }
    \label{fig:direct_est}
\end{figure}

\subsection{Direct Estimation}\label{subsec:direct_estimation}
Here we have tested direct estimation with simpler and finer weights (see Figure~\ref{fig:direct_est}). Recall that this approach does not use any samples so we did not expect high accuracy.
First, we observe that the plots for simpler and finer weights look very similar. Indeed, it turns out that in our dataset for less than 1\% of pairs of opinions, we either have $Com(A, B) \neq Impl(A,B)$ or $Neg(A, B) \neq Contr(A,B)$. For the remainder, both equalities hold, and hence, finer and simpler weights coincide. For this reason, in practice, we suggest replacing the time-consuming computation of finer weights with simpler weights and we will not distinguish between simpler and finer weights in the following.
We also observe that with increasing $Neg(A,B)$, the accuracy increases noticeably. Indeed there is an explanation for it. Recall the definition of $w_1$:
\[
	w_1 = \frac{2^{|A|} - 2^{|A| - |Neg(A, B)|}}{2^{|A|} - 1}\approx 1 - \frac{1}{2^{|Neg(A, B)|}}
\]
Thereby:
\begin{itemize}
	\item For $Neg(A, B) = 1$, we estimate the confirmation values for $1/2$ of all subsets correctly.
	\item For $Neg(A, B) = 2$, we estimate the confirmation values for $3/4$ of all subsets correctly.
	\item For $Neg(A, B) = 3$, we estimate the confirmation values for $7/8$ of all subsets correctly.
	\item Etc.
\end{itemize}
At this point, we want to remind you of one of the most prominent applications of mutual coherence - the ``Wahl-O-Mat'' system. In practice, there is normally a set of candidates such that each two of them certainly disagree on several statements. Thus, for any voter, her position will also necessarily differ from the positions of most candidates on several statements and hence, the direct estimation will yield an accurate result for the most $(voter, candidate)$-pairs. 
We also emphasize that in such an application the relative order of the mutual coherence values is important and not their absolute values.
Nevertheless, there will possibly still be a (small) set of candidates such that the value $Neg(voter, candidate)$ is small and a more accurate approximation technique is required. For this small number of pairs, one of the sample-based approaches can be applied.

\subsection{Sample-based Approaches}

\begin{figure}[t]
    \centering
    \includegraphics[scale = 0.12]{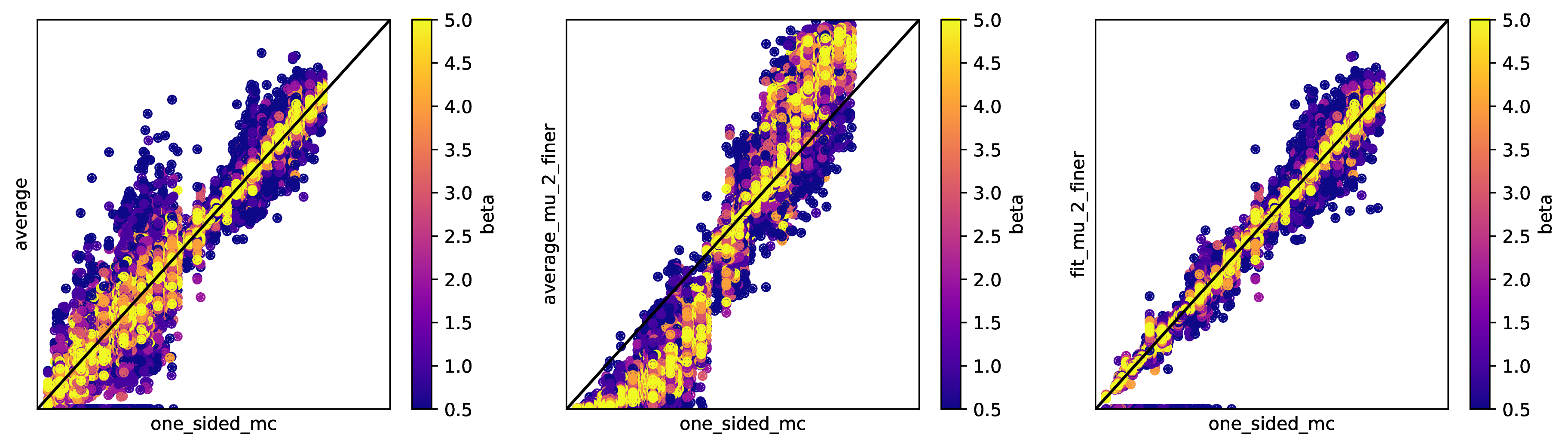}
    \caption{
    	Left: average-approach, middle: average-$\mu_2$-approach, right: fit-$\mu_2$-approach. The color corresponds to $\beta$.
    }
    \label{fig:sample_based}
\end{figure}

The results are demonstrated in Figure~\ref{fig:sample_based} where the color denotes $\beta$ (determines the number of used samples $\beta \cdot |A|$). The plots confirm our expectations. 

First of all, for the same number of samples, the average-approach results in the lowest accuracy since it only takes samples into account and does not utilize the weights. Especially on the negative side where $Neg(A,B) < 0$, the results of the fit-$\mu_2$-approach deviate from the line $x=y$ significantly less since the confirmation values are estimated correctly for more subsets than only for samples. 

As we have also supposed in the previous section, the average-$\mu_2$-approach tends to underestimate and overestimate the value on the negative and the positive side of the plot, respectively. This happens because the subsets with confirmation values -1 and 1 can be taken into account twice. This effect is most evident for yellow points corresponding to the largest value of $\beta$.

Finally, the results of the fit-$\mu_2$-approach fit the desired line $x = y$ and the error decreases with increasing $\beta$.

\begin{figure}[b]
    \centering
    \includegraphics[scale = 0.12]{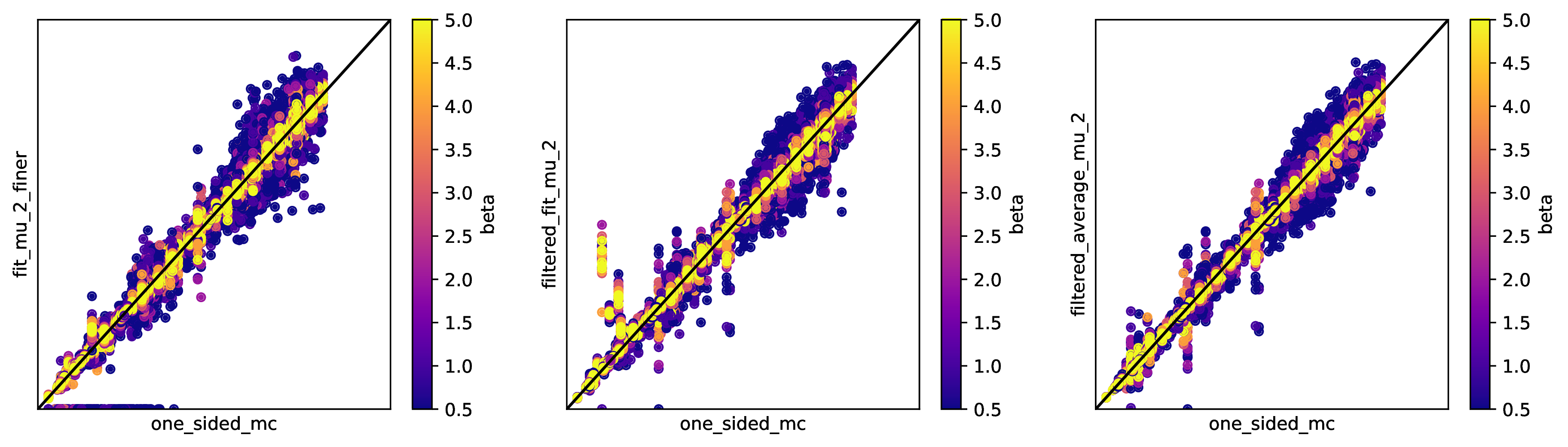}
    \caption{
    	Left: original fit-$\mu_2$-approach, middle: filtered fit-$\mu_2$-approach, right: filtered average-$\mu_2$-approach. The color corresponds to $\beta$.
    }
    \label{fig:filtered}
\end{figure}

\subsection{Filtering} \label{subsec:evaluation_filtering}

In Figure~\ref{fig:filtered}, we present the plots for filtered average-$\mu_2$- and filtered fit-$\mu_2$-approaches. 
For the comparison, the original fit-$\mu_2$-approach is plotted there too.
We expected that filtering would lead to higher accuracy because all samples are ``useful'' and no subset is taken into account twice. 
Although it is not obvious from the plot, the results confirm the expectation. For more details, we refer to Appendix~\ref{app:mse}.
Here only a short summary: the filtered average-$\mu_2$ approach leads to the highest accuracy, after that comes the filtered fit-$\mu_2$-approach and then, the original fit-$\mu_2$-approach. 
A further advantage of the filtered average-$\mu_2$ is that it does not require running the EM algorithm and hence, it is faster than the (filtered) fit-$\mu_2$-approach.
We have also tested robustness to make sure that the accuracy of the filtered-average-$\mu_2$-approach does not depend on parameters other than $\beta$. The plots can also be found in Appendix~\ref{app:robustness}.

\subsection{Summary}

In~Figure~\ref{fig:combi}, we provided the results of different approaches for $\beta = 3$. As we already stated, the average-$\mu_2$-approach leads to the highest accuracy and is the most efficient among the sample-based approaches, and hence, it is the unique leader and the main result of this paper. 

If only small computational power is available, the direct estimation can be applied instead. This approach becomes more accurate with increasing $Neg(A, B)$. In practice, two opinions differ on at least several statements and hence, a non-trivial approximation will be obtained.

\begin{figure}[t]
    \centering
    \includegraphics[scale = 0.12]{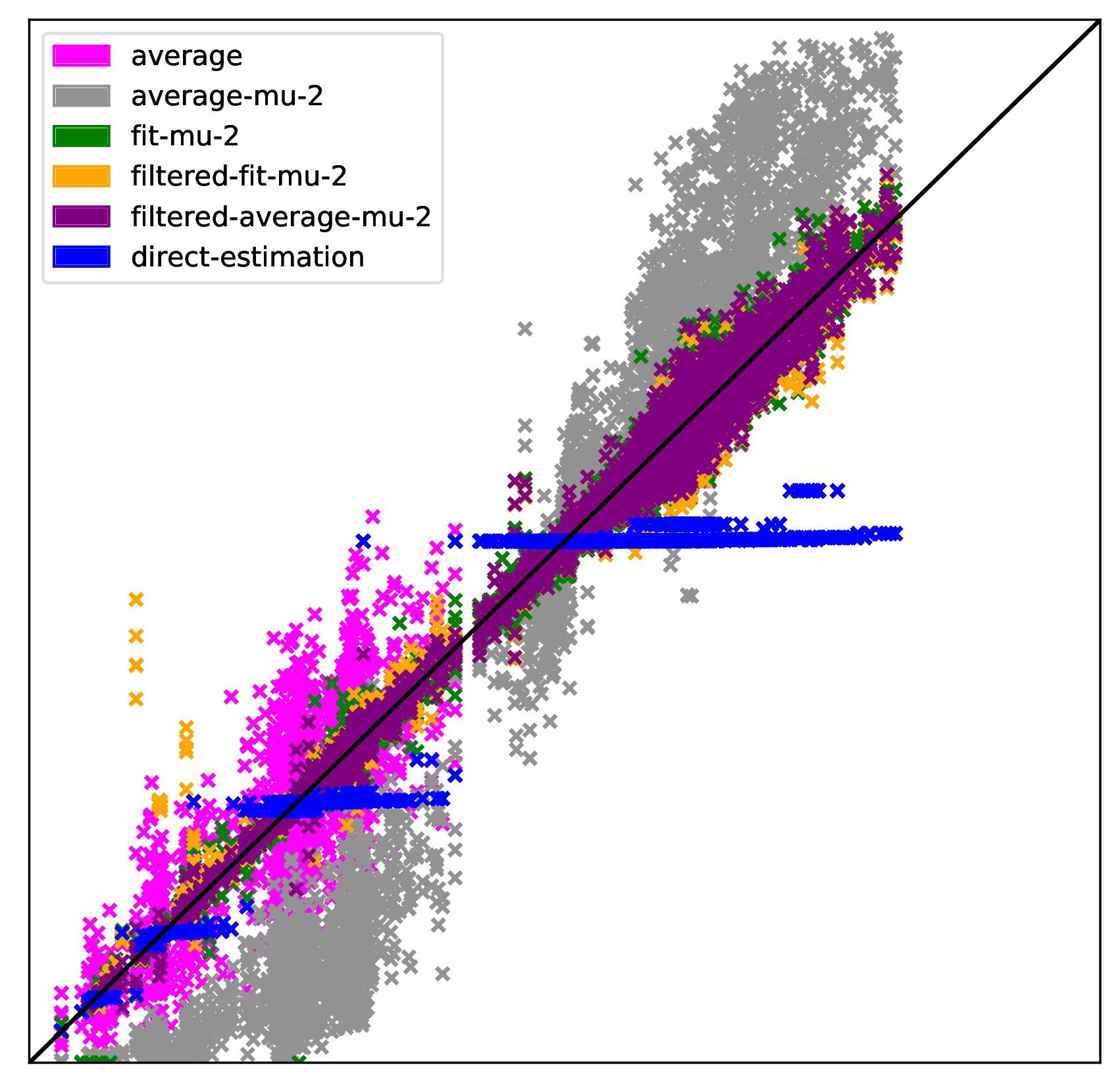}
    \caption{
    	The accuracy of the algorithms for $\beta = 3$. 
    }
    \label{fig:combi}
\end{figure}

\section{Conclusion}\label{sec:conclusion}

In this paper, we dealt with the efficient approximation of one measure of mutual coherence. First, we have encountered the problem of missing test data and as a side result, we have developed an algorithm for the creation of synthetic argument maps. This algorithm gets several parameters as the input determining the size, the density, and the form of created AMs. It also accepts the distribution of the number of premises of argument: this way, we can create synthetic AMs similar to a given real AM. 
The process is randomized, so if we run it multiple times with the same input, we obtain multiple argument maps with similar properties. 
After that, we generated argument maps and used them to test the approximation. 

To solve the main problem of this work (i.e., the approximation of mutual coherence), we have studied the definition of the measure of mutual coherence. First of all, the mutual coherence can be split into two symmetric parts - we call each of them the one-sided coherence. To get an efficient approximation of the mutual coherence, it is enough to get such an approximation for the one-sided coherence. This value is defined as the average confirmation value over all subsets of an opinion. 
The definition of a confirmation value contains an instance of the \#SAT problem. 
Thereby, for the naive computation, the exponential (in the size of the opinion) number of iterations is required and in every iteration, an instance of the SAT model counting problem needs to be solved. This, in turn, required exponential (in the size of the argument map) time in the worst-case. However, the state-of-the-art-\#SAT-solver are more efficient in practice. Our heuristics reduce the number of iterations and hence, the number of runs of a model counter.

The key observation was stated in Lemma~\ref{lemma:simple_bound} was that it is possible to efficiently determine (in general non-trivial) lower bounds on the number of subsets with confirmation values -1 and 1. 
This already reduces the number of needed runs of a model counter. 
In Lemma~\ref{lemma:improved_bounds}, we have improved the lower bounds but their computation became more time-consuming. However, we have later observed that in practice, both lemmas provide the same bounds and hence, we suggest using the results from Lemma~\ref{lemma:simple_bound}.

Next, we have observed that the distribution of confirmation values of an opinion pair can be modeled as a mixture of three Gaussians.
In our heuristics, we sample a linear (in the size of the opinion) number of subsets, apply a \#SAT solver to compute the confirmation values, and then estimate the parameters of the distribution in different ways. Finally, we compute the mean value of the distribution and this is then the desired approximation. 

One algorithm (direct estimation) is fully polynomial-time and if the number of statements on which two opinions differ is non-zero, then we already obtain a good approximation. And the accuracy increases if this number grows. We have also mentioned that in the Wahl-O-Mat application, for most opinion pairs, this number is typically large enough to achieve high accuracy. For the remaining opinion pairs, one of the more time-consuming sample-based approaches can be used. We have shown that with a filtering technique, we obtain an accurate algorithm whose running time requires only a linear (instead of exponential for the naive computation) number of runs of the model counter.

A possible direction for future research would be first to look for further subsets whose confirmation value can be computed efficiently.
\printbibliography

\newpage

\appendix
\section{Proof of Lemma~\ref{lemma:improved_bounds}}\label{app:proof_improved_bounds}
\begin{lemma}
Let $A, B \subset S$ be positions. Let 
	\[
		Cntr(A, B) \stackrel{\text{def}}{=} \{ x \in A \mid B \vDash_I \lnot \{x\}\}
	\]
	and
	\[
		Impl(A, B) \stackrel{\text{def}}{=} \{ x \in A \mid B \vDash_I \{x\}\}
	\]
	Then:
	\begin{enumerate}
		\item For every $X \subseteq A$ with $X \cap Cntr(A, B) \neq \emptyset$: $Conf(X, B) = -1$. And: 
		\[
			\big| \{ X \mid X \subseteq A, X \cap Cntr(A, B) \neq \emptyset \} \big| = 2^{|A|} - 2^{|A| - |Cntr(A,B)|} 
		\]
		\item For every $\emptyset \neq X \subseteq Impl(A, B)$: $Conf(X, B) = 1$. And:
		\[
			\big| \{ X \mid \emptyset \neq X \subseteq Impl(A, B) \} \big| = 2^{|Impl(A,B)|} - 1
		\]
	\end{enumerate}
\end{lemma}

\begin{proof}
We prove the claims in lemma separately.
\begin{enumerate}
\item Consider $X \subseteq A$ with $X \cap Cntr(A, B) \neq \emptyset$. 
There exists $x \in X \cap Cntr(A, B)$. 
For an arbitrary complete consistent position $P$ extending $B$, it holds $P(x) = false$ since $B \vDash_I \lnot \{x\}$. 
Because of $x \in X$, the position $P$ does not extend $X$. 
Thereby,~$B \vDash_I \lnot X$ and hence, $Conf(X, B) = -1$. 
The subsets of~$A$ which are disjoint from $Cntr(A,B)$ are exactly the subsets of $A \setminus Cntr(A, B)$ and hence, there are $2 ^ {|A| - |Cntr(A, B)|}$ of them.

\item Consider $\emptyset \neq X \subseteq Impl(A, B)$. Consider $x \in X \subseteq Impl(A, B)$. By the definition of $Impl(A, B)$, we have:~$B \vDash_I \lnot \{x\}$, i.e., every complete consistent position extending $B$ extends $\{x\}$ too. Since this holds for every $x \in X$, every complete consistent position extending $B$ extends $X$ too. Thereby, $B \vDash_I X$ and $Conf(X, B) = 1$. And there are $2^{|Impl(A, B)|} - 1$ non-empty subsets of~$Impl(A, B)$. \qed
\end{enumerate}
	
\end{proof}

\newpage

\section{Direct Estimation with Linear Slope on the Positive Side}\label{app:slope}
In Subsection~\ref{subsec:direct_estimation}, we have presented the accuracy of the direct estimation. We observe that although the approach is rather accurate on the negative side, a slope is missing on the positive side. For this reason, we have tried the following out. Instead of the proven lower bound 
\[
	w_3 = \frac{2^{|Com(A, B)|} - 1}{2^{|A|} - 1},
\]we have used the linear relation:
\[
	w_3' = \frac{Com(A, B)}{|A|}.
\]
However, this can lead to $w_1 + w_3' > 1$. Since we did not want to change $w_1$ too, for such pairs of opinions, we have not changed the value of $w_3$. The results of this adaptation can be found in Figure~\ref{fig:direct_estimation_vs_slope}, on the right. This approach is more accurate than the original direct estimation but we have no explanation for this heuristic so we have not included this in the main part of the paper.

\begin{figure}[t]
    \centering
    \includegraphics[scale = 0.5]{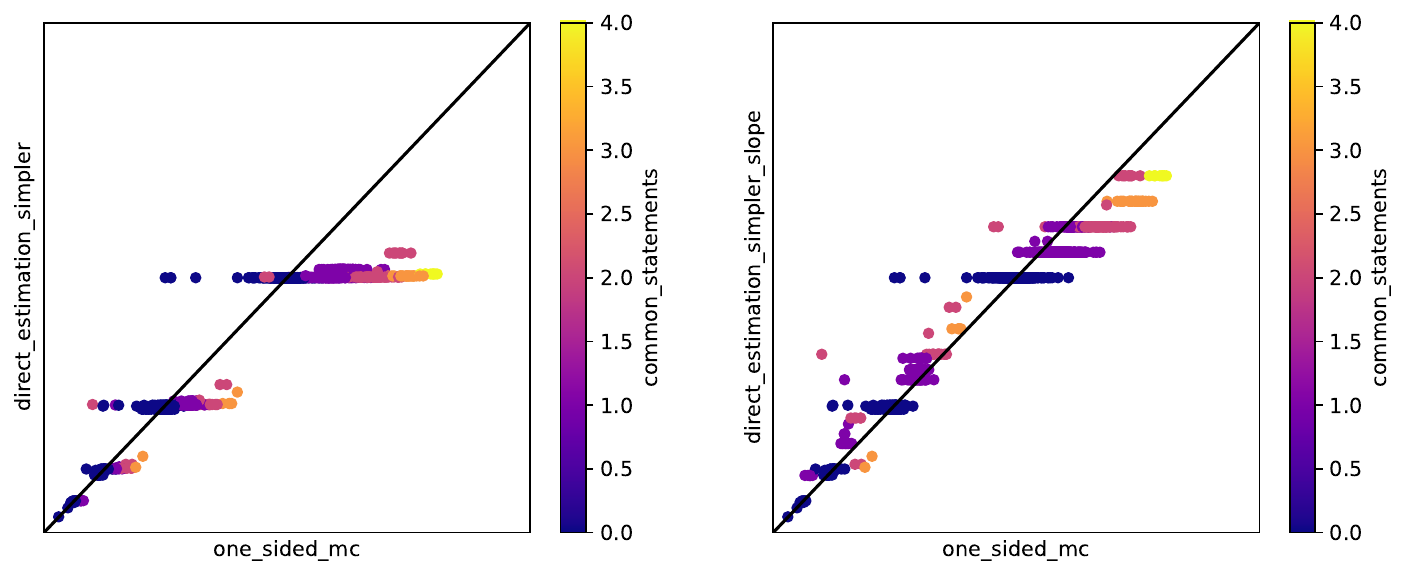}
    \caption{
    	Left: direct estimation, right: direct estimation with adapted weight $w_3'$. 
    }
    \label{fig:direct_estimation_vs_slope}
\end{figure}

\newpage

\section{Comparison of the Sample-based Approaches}\label{app:mse}
In Subsection~\ref{subsec:evaluation_filtering}, we claimed that the filtered-average-$\mu_2$ results in the best accuracy. However, the difference is not obvious from the plots provided there. For this reason, here we present the mean squared error of the approaches mentioned there depending on $\beta$. 

\begin{center}
	\begin{tabular}{|c|c|c|c|c|}
		\hline
		\backslashbox{$\beta$}{Approach} & fit-$\mu_2$ & filtered fit-$\mu_2$ & average $\mu_2$ & filtered average $\mu_2$ \\
		\hline
		0.5 & 0.0135 & 0.0052 & 0.0559 & 0.0035 \\
		\hline
		1 & 0.0050 & 0.0031 & 0.0419 & 0.0019 \\
		\hline
		2 & 0.0023 & 0.0020 & 0.0412 & 0.0010 \\
		\hline
		3 & 0.0016 & 0.0015 & 0.0425 & 0.0007 \\
		\hline
		4 & 0.0013 & 0.0012 & 0.0433 & 0.0005 \\
		\hline
		5 & 0.0011 & 0.0011 & 0.0439 & 0.0004 \\
		\hline
	\end{tabular}
	
\end{center}

\newpage

\section{Robustness Analysis of the Filtered $\mu_2$-approach}\label{app:robustness}

\begin{figure}[t]
    \centering
    \includegraphics[scale = 0.11]{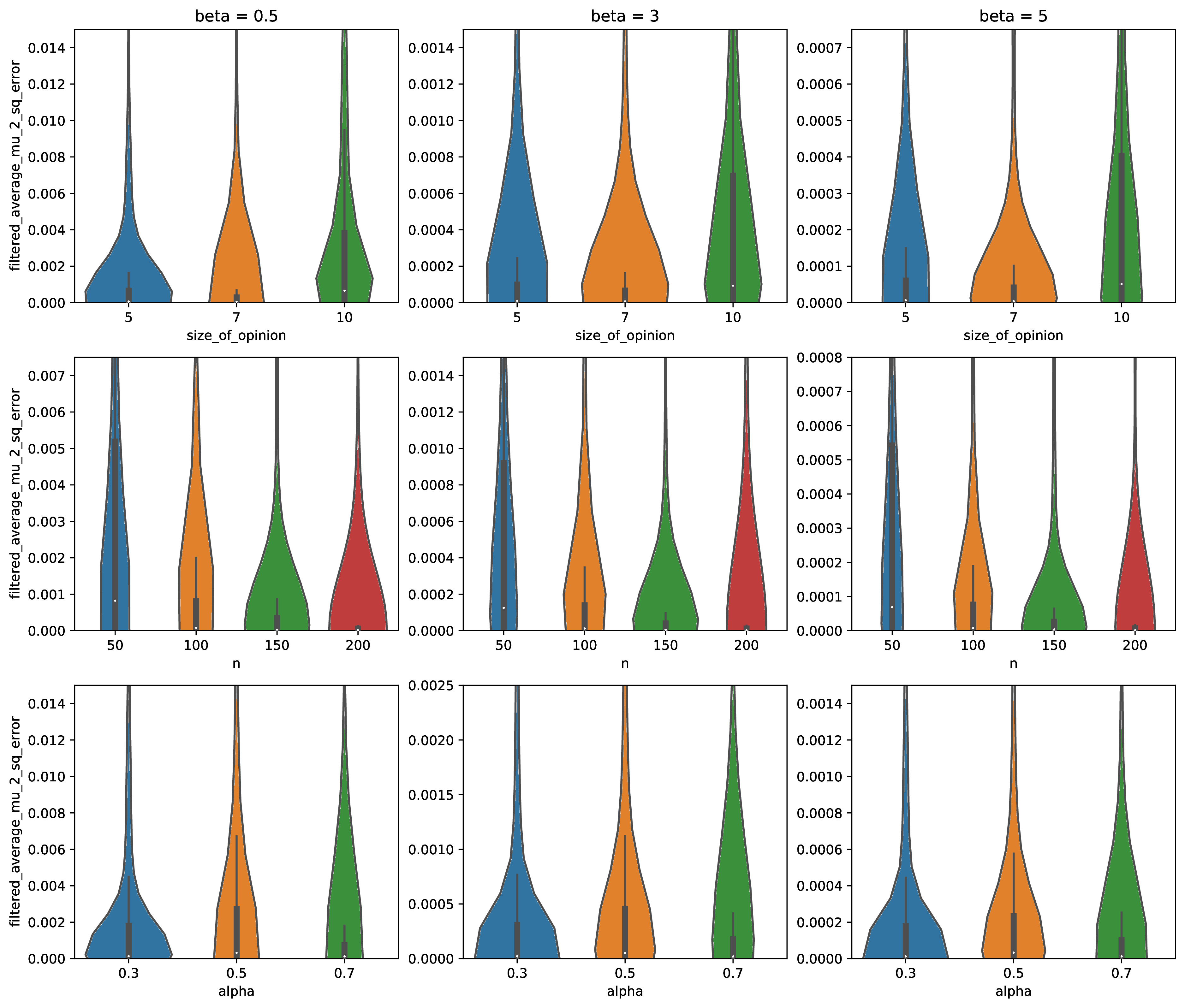}
    \caption{
    	Violin plots for the squared errors of filtered-average-$\mu_2$ approach.
    }
    \label{fig:violins}
\end{figure}

In Figure~\ref{fig:violins}, we have plotted the distribution of the squared errors of the filtered-average-$\mu_2$-approach depending on the size of the opinion (on top), size of the argument map (in the middle), and the parameter $\alpha$ (on bottom) for different values of $\beta$. 
The plots show that the accuracy of the approach does not decrease with increasing $\alpha$ and $n$. However, with increasing size of the opinion, the error increases very slowly.

\section{Generation of Argument Maps in Pseudocode}\label{appendix:pseudocode}

\IncMargin{1em}
\begin{algorithm}
\DontPrintSemicolon
\SetNoFillComment
\LinesNumbered
\SetKwData{Left}{left}\SetKwData{This}{this}\SetKwData{Up}{up}
\SetKwFunction{Union}{Union}\SetKwFunction{FindCompress}{FindCompress}
\SetKwFunction{AddVertices}{AddVertices}
\SetKwFunction{AddVertex}{AddVertex}
\SetKwFunction{PickFrom}{PickFrom}
\SetKwFunction{PickAccordingTo}{PickAccordingTo}
\SetKwFunction{NumberOfArguments}{NumberOfArguments}
\SetKwFunction{PickRandomInRange}{PickRandomInRange}
\SetKwFunction{PushBack}{PushBack}
\SetKwFunction{goto}{goto}
\SetKwFunction{Vertices}{Vertices}
\SetKwFunction{Level}{Level}
\SetKwFunction{Satisfiable}{Satisfiable}
\SetKwFunction{Size}{Size}
\SetKwFunction{NewArgumentVertex}{NewArgumentVertex}
\SetKwFunction{AddRedEdge}{AddRedEdge}
\SetKwFunction{AddGreenEdge}{AddGreenEdge}

\SetKwInOut{Input}{input}\SetKwInOut{Output}{output}
\Input{$n, k, \alpha, \psi, \gamma,distribution$}
\Output{An AM with $n$ statements and $\alpha \cdot n$ arguments}
\BlankLine

$G \leftarrow$ empty graph \;
$G$.\AddVertices($\{1, \dots, k\}$) \tcp*[f]{initial key statements} \; 
$clauses \leftarrow \langle \rangle$ \;

\While{$\Size(clauses) \leq \alpha \cdot n$}{
	$u \leftarrow$ \PickFrom($distribution$) \tcp*[f]{number of premises}\;
	$clause \leftarrow \langle \rangle$ \;
	$premises \leftarrow \langle \rangle$ \;
	$c \leftarrow$ \PickAccordingTo($G$.\Vertices, $v \mapsto \psi^{\Level(v)}$) \;
	\If{$\PickRandomInRange(0,1) > 0.5$} {
		$c \leftarrow \lnot c$ \;
	}
	$clause\leftarrow clause \cup c$\;
	\For{$j \leftarrow 1$ \KwTo $u-1$}{
		$p \leftarrow$ \PickAccordingTo($\langle 1 \dots n \rangle, v \mapsto \gamma^{\NumberOfArguments(v)}$) \; 
		$premises \leftarrow premises \cup p$ \;
		\If{$\PickRandomInRange(0,1) > 0.5$} {
			$p \leftarrow \lnot p$ \;
		}
		$clause\leftarrow clause \cup p$\;
	}
	\If{\Satisfiable($clauses \cup clause$)}{
		$clauses \leftarrow clauses \cup clause$ \;
		$ArgVertex \leftarrow \NewArgumentVertex ()$ \;
		\For{$p \in premises$}{
			G.\AddVertex(p) \tcp*[f]{add the premise to the AM if not already there} \;
			\eIf{$p$ is negative in $clause$}{
				G.\AddGreenEdge(p, ArgVertex) \;
			}{
				G.\AddRedEdge(p, ArgVertex) \;
			}
		}
		\eIf{$c$ is positive in $clause$}{
			G.\AddGreenEdge(ArgVertex, c) \;
		}{
			G.\AddRedEdge(ArgVertex, c) \;
		}
		
	}
}
\Return $(G, clauses)$ \;
\caption{Generation of a synthetic AM}\label{algo_disjdecomp}
\end{algorithm}\DecMargin{1em}

\end{document}